\documentclass[final, 12pt]{alt2025} % Anonymized submission
\usepackage{enumitem}
\usepackage{algorithm}
\usepackage{algorithmic}
\usepackage{dsfont}
\usepackage{pdfpages}
\usepackage{caption}
\usepackage{titletoc}

\newcommand\DoToC{%
  \startcontents
  \printcontents{}{1}{\hrule}
  \vspace{.5cm}
  \hrule
}

\usepackage{thmtools, thm-restate}
\usepackage{hyperref}
\usepackage{cleveref}
\usepackage{multicol}
\usepackage{wrapfig}
\usepackage{xspace}
\usepackage{booktabs}
\usepackage{tikz}
\usetikzlibrary{tikzmark,decorations.pathreplacing}
\newcommand*\circled[1]{\tikz[baseline=(char.base)]{
            \node[shape=circle,draw,inner sep=1pt] (char) {\small\textsf{#1}};}}
\usepackage{mathtools} % for equal def symbols, \coloneqq and \eqqcolon

\newcommand{\indicator}{\mathds{1}}

\newcommand{\bN}{\mathbb{N}}
\newcommand{\bE}{\mathbb{E}}
\newcommand{\bP}{\mathbb{P}}
\newcommand{\bR}{\mathbb{R}}
\newcommand{\cB}{\mathcal{B}}
\newcommand{\cD}{\mathcal{D}}
\newcommand{\cE}{\mathcal{E}}

\newcommand{\cH}{\mathcal{H}}
\newcommand{\cG}{\mathcal{G}}
\newcommand{\cI}{\mathcal{I}}
\newcommand{\cJ}{\mathcal{J}}
\newcommand{\cN}{\mathcal{N}}
\newcommand{\cP}{\mathcal{P}}

\newcommand{\OPT}{A^*}

\newcommand{\DGETC}{\texttt{DG-ETC}\xspace}
\newcommand{\DG}{\texttt{DG}\xspace}
\newcommand{\RGL}{\texttt{RGL}\xspace}

\newcommand{\SamplingSet}{\texttt{DG-Sp}\xspace}
\newcommand{\UpdateExplo}{\texttt{UpdExp}\xspace}

\newcounter{defCounter}
\newtheorem{ddefinition}[defCounter]{Definition}

\newcounter{thCounter}
\newtheorem{ttheorem}[thCounter]{Theorem}
\newtheorem*{ttheorem*}{Theorem}

\newcounter{lemCounter}
\newtheorem{llemma}[lemCounter]{Lemma}

\newcounter{remCounter}
\newtheorem{rremark}[remCounter]{Remark}

\newtheorem*{rremark*}{Remark}
\newtheorem*{rremarks*}{Remarks}

\newtheorem*{eexample*}{Example}

\title{Logarithmic Regret for Unconstrained Submodular Maximization Stochastic Bandit}
\usepackage{times}
\altauthor{%
 \Name{Julien Zhou} \Email{julien.zhou@inria.fr}\\
 \addr Criteo AI Lab, Paris, France\\
    Univ. Grenoble Alpes, Inria, CNRS, Grenoble INP, LJK, 38000 Grenoble, France 
 \AND
 \Name{Pierre Gaillard} \Email{pierre.gaillard@inria.fr}\\
 \addr Univ. Grenoble Alpes, Inria, CNRS, Grenoble INP, LJK, 38000 Grenoble, France 
 \AND
 \Name{Thibaud Rahier} \Email{t.rahier@criteo.com}\\
 \addr Criteo AI Lab, Paris, France
 \AND
 \Name{Julyan Arbel} \Email{julyan.arbel@inria.fr}\\
 \addr Univ. Grenoble Alpes, Inria, CNRS, Grenoble INP, LJK, 38000 Grenoble, France 
}

\begin{document}

\maketitle

\begin{abstract}%
    We address the \emph{online unconstrained submodular maximization problem} (Online USM), in a setting with \emph{stochastic bandit feedback}. In this framework, a decision-maker receives noisy rewards from a non monotone submodular function taking values in a known bounded interval. This paper proposes \emph{Double-Greedy - Explore-then-Commit} (\DGETC), adapting the Double-Greedy approach from the offline and online full-information settings. \DGETC satisfies a $O(d\log(dT))$ problem-dependent upper bound for the $1/2$-approximate pseudo-regret, as well as a  $O(dT^{2/3}\log(dT)^{1/3})$ problem-free one at the same time, outperforming existing approaches. In particular, we introduce a problem-dependent notion of hardness characterizing the transition between logarithmic and polynomial regime for the upper bounds.
\end{abstract}


\section{Introduction}
\label{sec:introduction}

\subsection{Context and problem formulation}

Several real-world settings can be cast as combinatorial optimization problems over a finite set. Without some assumptions on the utility function to be maximized and/or the constraints to be satisfied, such problems cannot be solved in polynomial time. In practice, different types of assumptions and constraints can be introduced to make these problems manageable, even approximately. One can, for example, assume the utility to be linear, but in some cases even this already strong assumption can be helpless to make the problem easier.

This paper focuses on the cases where we maximize a submodular set-function, meaning that it satisfies a “diminishing marginal gains" property. We consider the unconstrained setting, where the whole combinatorial super-set is available and the utlitity may be nonmonotone (if we know that it is monotone, the solution is straightforward, being either the full or the empty set). We also place ourselves in a stochastic (combinatorial) bandit setting, where a \emph{decision-maker / player} chooses different sets in sequential rounds, and receives noisy rewards. In this framework, a classic challenge is to balance exploration and exploitation, but the problem of managing the combinatorial complexity of the action set is stacked over it. In particular, a good strategy should efficiently leverage the underlying structure of the reward -- submodularity in this case -- by monitoring relevant quantities.

\paragraph{Problem formulation and assumptions.} 
We consider a finite set of $d\in\bN^*$ \emph{items} $\cD$. The player has access to \emph{actions} from the whole superset $\cP(\cD)$ and plays for an \emph{horizon} of $T\in\bN^*$ rounds. 
The player receives noisy \emph{rewards} from a \emph{non monotone submodular} set-function $f: \cP(\cD)\rightarrow [0, c]$ with $c>0$. At each round $t\in[T]=\{1, \dots, T\}$, the player chooses an action $A_t\in\cP(\cD)$ and receives 
\begin{align}
    Z_t = f(A_t) + \eta_t\,,
\end{align}
where $\eta_t$ is a random variable. Let $\sigma>0$ (known), we assume that $\eta_t$ is $\sigma^2$-sub-Gaussian conditionally to the past (including the possibly random process generating $A_t$).

The algorithms that we study in this paper all consider items sequentially. For convenience, we identify $\cD$ with $[d]$ and assume an arbitrary ordering, but a player with prior knowledge could try to optimize over permutations.

\paragraph{$1/2$-Approximate pseudo-regret minimization.} The objective of the player is to maximize its cumulative rewards over the $T$ rounds. As it is common in the bandit literature, we look instead at a pseudo-regret, neglecting the contribution of the noise $(\eta_t)_{t\in[T]}$. Besides, rather than looking at the exact pseudo-regret, we minimize an $1/2$-approximation defined as 
\begin{align}
\label{eq:regret}
    R_T = \sum_{t=1}^T\Big[\frac{1}{2}f(A^*) - f(A_t)\Big]\,,
\end{align}
where $A^* \in \arg\max_{A\subseteq\cD}\big\{f(A)\big\}$\,.

Considering approximate regrets is usual in settings where we have access to an \emph{oracle} solving the offline optimization approximately \citep{chen2013combinatorial}. In our framework, the $1/2$ factor comes from the impossibility of solving the offline unconstrained submodular maximization problem (USM), with a competitive ratio better than $1/2$, using a polynomial number of calls \citep{feige2011maximizing}. 

In the following, if not specified, the expressions “pseudo-regret" or just “regret" refer to the \emph{$1/2$-approximate pseudo-regret}.

\subsection{Contributions}

We propose a novel algorithm \emph{Double-Greedy - Explore-then-commit} (\DGETC) for the online unconstrained submodular maximization problem (Online USM), with stochastic bandit feedback (Section~\ref{sec:algorithm}). We introduce a new notion of \emph{hardness} for this problem (Section~\ref{sec:hardness}), and prove that \DGETC satisfies both a logarithmic problem-dependent (hardness-dependent) upper bound for the $1/2$-approximate pseudo-regret, as well as a worst-case $O(dT^{2/3}\log(dT)^{1/3})$ upper bound (Sections~\ref{sec:regret} and~\ref{sec:analysis}). These bounds are satisfied both with high-probability and in expectation (Theorem~\ref{th:DGETC}), and rely on the stationarity of the stochastic setting. Asymptotically, \DGETC allocates a logarithmic, hardness-dependent, number of rounds to the design of a strategy that compensates the randomness errors with per-round negative losses (therefore, with gains). In practice, \DGETC exploits the looseness of the $1/2$-approximation ratio in non-adversarial cases, and we argue that this kind of strategy could also be applied to other settings involving approximations.

\subsection{Related works}
\label{sec:related}

In this section, we mention the closest related works, concerning combinatorial bandits, offline (unconstrained) submodular maximization, as well as its online and bandit versions. Supplementary discussions with other lines of work can be found Appendix~\ref{app:related} (offline and online minimization, constrained maximization, and other online maximization problems).

\paragraph{Combinatorial bandits.} The recent monograph by \cite{lattimore2020bandit} makes an extensive study of bandit problems. We are more particularly interested in settings where the action space is combinatorial and too big to be explored in its entirety. \citet{chen2013combinatorial} (extended by \citet{chen2016combinatorial}) introduces the stochastic semi-bandit framework, and derives results for approximate pseudo-regrets and smooth, monotone aggregation functions. When the aggregation is linear, the leading factors in the regret upper bounds have been refined in several subsequent works \citep{kveton2015tight, degenne2016combinatorial, perrault2020statistical, zhou2024efficientoptimalcovarianceadaptivealgorithms}. A matching adversarial semi-bandit setting has also been explored \citep{ito2021hybrid, neu2014online}. While the player gets one feedback per chosen item in the semi-bandit setting, the full-bandit (or just “bandit") setting with a single feedback per action is more challenging. If the aggregation remains linear, one could see the problem as a linear bandit and use the corresponding methods, as long as the offline problem can be solved \citep{abbasi2011improved, bubeck2012towards}. However, Considering a nonlinear aggregation function with a full-bandit feedback remains challenging without further assumptions \citep{han2021adversarial}.

\paragraph{Unconstrained sumbodular maximization (USM).} Several systems can be modeled with a submodular structure in various fields, including economics, game theory and combinatorial optimization. As it shares properties similar to both convexity and concavity in continuous optimization \citep{Lovász1983}, both viewpoints are of interest. The monograph by \citet{bach2013learning} details various cases where submodular set-functions appear and highlights the parallels between submodular minimization and convex optimization. While minimization can be solved in polynomial time, maximization is more challenging and can in general only be solved approximately \citep{feige2011maximizing}. A $(1-1/e)$-approximation is possible in the cardinally-constrained monotone case \citep{nemhauser1978analysis}, but the unconstrained non monotone setting can only be solved up to a $1/2$ approximation ratio \citep{feige2011maximizing}. In particular, \cite{Buchbinder2012} provides a linear-time approach reaching this ratio, closing the gap between upper and lower bounds. 

\paragraph{Online USM with full-information and bandit feedback.} Following the results from \citet{buchbinder2014submodular}, \cite{roughgarden2018optimalalgorithmonlineunconstrained} studies the particular case of non monotone, unconstrained maximization in the online adversarial full-information setting and provides an algorithm satisfying a $O(d\sqrt{T})$ regret upper bound. \citet{harvey2020improved} manages to gain a $\sqrt{d}$ factor by using tools related to online dual averaging and Blackwell approachability. \cite{fourati2023randomizedgreedylearningnonmonotone} considers a stochastic bandit setting, and proposes an Explore-then-Commit type algorithm satisfying a $O(dT^{2/3}\log(T)^{1/3})$ regret upper bound. However, \cite{niazadeh2021online} claims a similar $O(dT^{2/3})$ in an adversarial bandit setting. As the latter framework seems significantly more difficult, one may reasonably wonder if better guarantees can be satisfied in the stochastic setting. We answer this question positively and propose an algorithm satisfying both logarithmic problem-dependent and $O(d(T\log(dT))^{2/3})$ problem-free upper bounds.

\section{Preliminary}

In this section, we introduce submodularity, and remind the approach of the Double-Greedy (\DG) algorithm \citep{Buchbinder2012} on which our \DGETC is based.

\subsection{Submodularity}

Submodularity is a “diminishing marginal gains" property, it is formally defined as follows.

\begin{restatable}[Submodularity]{ddefinition}{defSubmodularity}
\label{def:sudmodular}
    Let $\cD$ be a finite set and $c>0$. A set-function $f: \cP(\cD) \rightarrow [0, c]$ is said to be (bounded) \emph{submodular} if, equivalently,
    \begin{itemize}[nosep]
        \item For all $A\subseteq B\subseteq\cD$ and $i\in\cD$\,,\qquad  
        $
            \textstyle f(B\cup \{i\}) - f(B) \leq f(A\cup\{i\})-f(A)\,;
        $
        \item For all $(A,\ B)\in\cP(\cD)\times\cP(\cD)$\,,\quad 
        $
            \textstyle f(A\cup B) + f(A\cap B) \leq f(A)+f(B)\,.
        $
    \end{itemize}
    Besides, $f$ is said to be \emph{monotone} if for all $A\subseteq B\subseteq \cD$, $f(A)\leq f(B)$. Otherwise, we say that $f$ is \emph{non monotone}.
\end{restatable}

\subsection{Double-Greedy for USM}

In this section, we outline Double-Greedy Algorithm (\DG, Algorithm~\ref{alg:DG}) from \cite{Buchbinder2012}.

\begin{wrapfigure}{r}{0.5\linewidth}
\vspace{-1cm}
    \begin{minipage}{\linewidth}
        \begin{algorithm}[H]
            \caption{Double-Greedy (\DG from \citealp{Buchbinder2012})}
            \label{alg:DG}
            \begin{algorithmic}[1]
                \STATE \textbf{Inputs: } $\cD$\,.
                \STATE $(X_0,\ Y_0) \leftarrow (\emptyset,\ \cD)$\,.
                \FOR{$i=1,\dots d$}
                    \STATE $\alpha_i\leftarrow f(X_{i-1}\cup\{i\})-f(X_{i-1})$\,.
                    \STATE $\beta_i\leftarrow f(Y_{i-1}\setminus\{i\})-f(Y_{i-1})$\,.
                    \STATE $p_i \leftarrow \frac{\max\{\alpha_{i},0\}}{\max\{\alpha_{i},0\} + \max\{\beta_{i},0\}}$\,.
                    \STATE $K_i \sim \cB(p_i)$\,.
                    \IF{$K_i$}
                        \STATE $(X_i,\ Y_i) \leftarrow (X_{i-1}\cup\{i\},\ Y_{i-1})$\,.
                    \ELSE
                        \STATE $(X_i,\ Y_i) \leftarrow (X_{i-1},\ Y_{i-1}\setminus\{i\})$\,.
                    \ENDIF
                \ENDFOR
                \STATE \textbf{Return: } $X_d\subseteq\cD$\,.
            \end{algorithmic}
        \end{algorithm}
    \end{minipage}
\vspace{-.5cm}
\end{wrapfigure}

When maximizing a nonmonotone submodular function $f$, \DG works in $d$ steps (one per item) and considers the items sequentially, th order being chosen by the user beforehand.

It first initializes a pair of sets $X_0=\emptyset$ and $Y_0=\cD$ respectively as the empty set and the full set, and then modifies them sequentially. 

At each step $i\in[d]$, \DG looks at the “marginal gains" $\alpha_i$ and $\beta_i$ respectively corresponding to adding item $i$ to $X_{i-1}$ or removing it from $Y_{i-1}$ [Line 4-5]. It makes the decision of either adding or removing the item by sampling a Bernoulli random variable $K_i$ with parameter $p_i$, defined from the positive part of $\alpha_i$ and $\beta_i$ [Line 6-7]. After the $d$-th and last step, \DG returns the set $X_d$, which is identical to $Y_d$ by construction [Line 14].

Overall, $\DG$ requires $4d$ calls to $f$ and satisfies the following guarantee.

\begin{ttheorem}[\citealp{Buchbinder2012}, Theorem I.2.]
\label{th:DG}
    Let $\cD$ be a finite set. Algorithm \DG returns a set $S$ such that 
    \begin{equation*}
        \bE\big[f(S)\big]\geq \frac{1}{2}f(\OPT)\,.
    \end{equation*}
\end{ttheorem}

The result being in expectation, one can repeatedly run \DG to obtain an acceptable set with a high enough probability. In particular, we prove the following proposition in Appendix~\ref{app:highProbaDG}.

\begin{restatable}{pproposition}{highProbaDG}
\label{th:highProbaDG}
    Let $\cD$ be a finite set, $\delta>0$ and $T\in\bN^*$ such that $T>2\log(1/\delta)$\,. If $(S_i)_{i\in[T]}$ is the sequence of sets obtained by running independently  $T$ times Algorithm \DG, then
    \begin{align*}
        \max_{i\in[T]}f(S_i) > \Big(\frac{1}{2}-\frac{\log(1/\delta)}{T}\Big)f(\OPT)\,, 
        \qquad \text{w.p.} \quad 1-\delta\,.
    \end{align*}
\end{restatable}

\paragraph{Stochastic bandit setting. } In our setting, using \DG directly is not possible as we do not have access to the marginal gains $\alpha_i$ and $\beta_i$ but only to noisy estimates. To overcome this difficulty, \cite{fourati2023randomizedgreedylearningnonmonotone} propose the \emph{Randomized Greedy Learning} (\RGL) algorithm, an \emph{Explore-then-Commit} strategy satisfying a $O(dT^{2/3}\log(T)^{1/3})$ expected regret upper bound. Similarly to \DG, \RGL~works in $d$ steps, one per item, each lasting $T^{2/3}\log(T)^{1/3}$ rounds. During the $i$-th step, \RGL~estimates the coefficients $\alpha_i$ and $\beta_i$\,, chooses a set $X_i$ (and $Y_i$) and move on to the next item. After $dT^{2/3}\log(T)^{1/3}$ exploration rounds, \RGL~commits to the last chosen set $X_d$\,.

However, we argue that \RGL~explores too much, and that logarithmic, problem-dependent regret upper bounds can be obtained both in expectation and with high-probability, 

\section{Full-bandit feedback algorithm: \emph{Double-Greedy - Explore-then-Commit} (\DGETC)}
\label{sec:algorithm}

In this section, we propose \emph{Double-Greedy - Explore-then-Commit} (\DGETC), a novel algorithm for unconstrained submodular maximization (USM) with stochastic full bandit feedback. \DGETC builds on insights from \cite{Buchbinder2012}, \cite{roughgarden2018optimalalgorithmonlineunconstrained} and \cite{harvey2020improved}. We present  the theoretical guarantees of \DGETC in Section~\ref{sec:theory}, which outperform existing approaches for this setting.

In the following, the word \textit{round} refers to a single increment of time $t$, the word \textit{step} refers to the per-item exploration steps (containing several rounds) and the word \textit{phase} refers to the exploration/exploitation phases (the exploration phase containing one step per item).

\subsection{Algorithms presentation}

\DGETC~is presented in Algorithm~\ref{alg:DGETC}, and is built on two subroutines: \SamplingSet~(Algorithm~\ref{alg:SamplingSet}) to sample sets, and \UpdateExplo~(Algorithm~\ref{alg:UpdateExplo}) to update exploration parameters.

\paragraph{\emph{Double-Greedy - Explore-then-Commit} (\DGETC, Algorithm~\ref{alg:DGETC}).} 
Algorithm \DGETC is an algorithm implementing an \emph{Explore-then-Commit} type strategy. It takes as inputs the set of items $\cD$, the range of $f$ $c>0$ , the sub-Gaussian parameter of the noise $\sigma>0$, as well as the horizon $T\in\bN^*$ and a confidence level $\delta\in(0,1)$\,.  
It first performs $d$ exploration steps (one per item in $\cD$) [Lines 12 to 26], each lasting at most $4\tau_{\max}$ rounds where 
\begin{align}
\label{eq:tauMax}
    \textstyle \tau_{\max}= T^{2/3}\log(dT)^{1/3}\,.
\end{align}
Contrarily to \RGL ~\citep{fourati2023randomizedgreedylearningnonmonotone}, the duration of each exploration step is problem-adaptive, and can be considerably smaller than the aforementioned worst case (See Section~\ref{sec:explorationBound}). It then spends the rest of the rounds (at least $T^{1/3}\log(dT)^{2/3}$ ones) exploiting the collected information [Lines 27 to 32]. During this phase, it does not play a fixed set, but repeatedly samples random sets based on $d$ Bernoulli random variables with parameters $(p_j)_{j\in[d]}$ determined during the exploration phase.

\begin{algorithm}
    \caption{Double-Greedy - Explore-then-Commit (\DGETC)}
    \begin{algorithmic}[1]
    \label{alg:DGETC}
    \begin{multicols}{2}
        \STATE \textbf{Inputs:} $\cD,\ c>0,\ \sigma>0,\ \delta>0,\ T\in\bN^*$\,.
        \STATE \texttt{/* Instantiating */}
        \STATE $d\leftarrow |\cD|$\,.
        \STATE Instantiate $g_{T, \delta}$ and $\tau_{\max }$ with \eqref{eq:gTdelta} and \eqref{eq:tauMax}.
        \STATE Instantiate \UpdateExplo with $g_{T, \delta}$ and $\tau_{\max }$\,.
        \STATE \texttt{/* Initialisation */}
        \STATE $(t,\ i) \leftarrow (1,\ 1)$\,.
        \STATE $(\hat \alpha_{j},\ \hat \beta_{j})_{j\in[d]} \leftarrow 0$\,.
        \STATE $(p_{j})_{j\in[d]}\leftarrow1/2$\,.
        \STATE $(\tau_j)_{j\in[d]}\leftarrow0$\,.
        \STATE \texttt{/* Exploration phase */}
        \WHILE{$i\leq d$}
            \STATE \texttt{/* 4 rounds exploration */}
            \STATE $(X_{i-1},\ Y_{i-1}) \leftarrow\SamplingSet\big(\cD,\ (p_{j})_{j},\ i\big)$\,.
            \STATE \underline{Play}: 
            \STATE \hspace{0.2cm}$A_t \leftarrow X_{i-1},\ A_{t+1} \leftarrow X_{i-1} \cup \{i\}$\,,
            \STATE \hspace{0.2cm}$A_{t+2} \leftarrow Y_{i-1},\ A_{t+3} \leftarrow Y_{i-1} \setminus \{i\}$\,.
            \vspace{.4cm}
            \STATE \underline{Receive}:
            \STATE \hspace{0.2cm}$Z_t\,,\  Z_{t+1}\,,\ Z_{t+1}\,,\ Z_{t+3}$\,.
            \STATE \underline{Update}:
            \STATE \hspace{0.2cm}$\hat \alpha_{i} \leftarrow \big(\tau_{i} \hat \alpha_{i} + (Z_{t+1}-Z_t)\big)/(\tau_{i}+1)$\,,
            \STATE \hspace{0.1cm} $\hat \beta_{i} \leftarrow \big(\tau_{i}\hat \beta_{i} +(Z_{t+3}-Z_{t+2})\big)/(\tau_{i}+1)$\,,
            \STATE \hspace{0.2cm}$\tau_{i}\leftarrow \tau_{i}\ +1$\,,
            \STATE \hspace{0.2cm}$(p_{i},\ i) \leftarrow \UpdateExplo\big(i,\ (\hat \alpha_{i},\ \hat \beta_{i}),\ \tau_{i}\big)$\,,
            \STATE \hspace{.2cm}$t\leftarrow t+4$\,.
        \ENDWHILE
        \STATE \texttt{/* Exploitation phase */}
        \WHILE{$t\leq T$}
            \STATE $(X_d,\ Y_d) \leftarrow \SamplingSet\big(\cD,\ (p_j)_{j},\ i\big)$\,.
            \STATE \underline{Play}: $A_t\leftarrow X_d$\,.
            \STATE \underline{Update}: $t\leftarrow t+1$\,.
        \ENDWHILE
    \end{multicols}
    \end{algorithmic}
\end{algorithm}

\begin{wrapfigure}{r}{0.5\linewidth}
    \centering
    \scalebox{.5}{
\begin{tikzpicture}%[x=1.2cm, y=0.8cm, >=stealth]
% Define parameters
\def\squareSize{0.4cm}  % Size of each square
\def\barLength{15}      % Number of squares per bar
\def\rowSpacing{1cm}    % Horizontal space between bars in each row
\def\rowHeight{1cm}     % Vertical spacing between rows

% Add captions above columns
\node at (7 * \squareSize, 0.9cm) {\textcolor{blue!80}{$X_{j}$}}; % Centered above the first bar
\node at (7 * \squareSize + \rowSpacing + \barLength * \squareSize, 0.9cm) {\textcolor{red!80}{$Y_{j}$}}; % Centered above the second bar

    % Set colors for each row
        \def\leftColor{blue!20} % Pale blue for first row
        \def\rightColor{red!80} % Red for first row

\foreach \i in {0,1,2,3,4} {
    % Add index at the left of each row
    \node at (-1cm, -\i * \rowHeight + 0.2cm) {$j=\i$};
}
    \node at (-1cm, -6 * \rowHeight + 0.2cm) {$j=d$};

% Loop to create 5 rows
\foreach \i in {0,1,2,3,4, 6} {
    % First rectangle bar (left)
    \foreach \j in {0,1,...,14} {
        \draw[fill=\leftColor] (\j * \squareSize, -\i * \rowHeight) rectangle ++(\squareSize, \squareSize);
    }
    
    % Second rectangle bar (right)
    \foreach \j in {0,1,...,14} {
        \draw[fill=\rightColor] (\j * \squareSize + \rowSpacing + \barLength * \squareSize, -\i * \rowHeight) rectangle ++(\squareSize, \squareSize);
    }
\foreach \i in {1,2,3,4,6} {
     \draw[fill=blue!80] (0 * \squareSize, -\i * \rowHeight) rectangle ++(\squareSize, \squareSize);
}
\foreach \i in {2,3,4,6} {
     \draw[fill=red!20] (1 * \squareSize + \rowSpacing + \barLength * \squareSize, -\i * \rowHeight) rectangle ++(\squareSize, \squareSize);
}
\foreach \i in {3,4,6} {
     \draw[fill=blue!80] (2 * \squareSize, -\i * \rowHeight) rectangle ++(\squareSize, \squareSize);
}
\foreach \i in {4,6} {
     \draw[fill=red!20] (3 * \squareSize + \rowSpacing + \barLength * \squareSize, -\i * \rowHeight) rectangle ++(\squareSize, \squareSize);
}
\foreach \j in {4, 5, 7, 8, 9, 11, 14} {
        \draw[fill=blue!80] (\j * \squareSize, -6 * \rowHeight) rectangle ++(\squareSize, \squareSize);
    }
\foreach \j in {6, 10, 12, 13} {
        \draw[fill=red!20] (\j * \squareSize + \rowSpacing + \barLength * \squareSize, -6 * \rowHeight) rectangle ++(\squareSize, \squareSize);
    }
}
\end{tikzpicture}
}
    \caption{Example of sampling from \SamplingSet, for $i=d+1$ and $(K_{j})_{j\in[d]}=(1, 0, 1, 0, \dots 1)$\,.}
    \label{fig:sampleBox}
\end{wrapfigure}

\paragraph{\emph{Double-Greedy - Sampling} (\SamplingSet, Algorithm~\ref{alg:SamplingSet}).} Both exploration and exploitation phases rely on the \SamplingSet~subroutine [Lines 14 and 29 in Algorithm~\ref{alg:DGETC}], which is a variation of \DG from \cite{Buchbinder2012} (Algorithm~\ref{alg:DG}). \SamplingSet~relies on the parameters $(p_j)_{j\in[d]}$ provided by the meta-algorithm \DGETC, which also provides a step $i\in\{1, \dots, d, d+1\}$ before which \SamplingSet~should stop. Like \DG, it begins by initializing two sets $X_0$ and $Y_0$ as the empty and the full sets. Then it iterates over the parameters $(p_j)_{j\in[d]}$ and proceeds to either add (to $X_{j-1}$) or remove (from $Y_{j-1}$) item $j$ in order to create $(X_j, Y_j)_{j<i}$\,, by sampling Bernoulli random variables. At the end, \SamplingSet~returns $(X_{i-1},\ Y_{i-1})$ and \DGETC~then decides to either collect information when $i\leq d$ or exploit when $i=d+1$\,. An example of sampling from \SamplingSet~is illustrated in Figure \ref{fig:sampleBox}.

\paragraph{Exploration update for \DGETC (\UpdateExplo, Algorithm~\ref{alg:UpdateExplo}).} During the exploration, \DGETC makes calls to Subroutine \UpdateExplo~[Line 24 in Algorithm~\ref{alg:DGETC}]. The latter takes as inputs the index of the current step $i$\,, estimates of the marginal gains $(\alpha,\ \beta)$ and the current values of $\tau$ for item $i$\,. The objective of \UpdateExplo~is to check if we can determine an adequate Bernoulli parameter $p$ for item $i$ and/or if the exploration has lasted too long (if $\tau\geq\tau_{\max }$). In both those cases, \UpdateExplo~returns an adequate parameter $p$ and index $i+1$ to tell \DGETC to switch to the next item. Otherwise, $p$ stays the default $1/2$ and \UpdateExplo~returns current index $i$\,.

\begin{center}
    \begin{minipage}{0.46\linewidth}
    \begin{algorithm}[H]
        \caption{{Double-Greedy - Sampling}\hfill\\
        (\SamplingSet)}
        \begin{algorithmic}[1]
         \label{alg:SamplingSet}
            \STATE \textbf{Inputs:} $\cD,\ (p_j)\in[0, 1]^d,\ i\in[d+1]$.
            \STATE $(X_{0},\ Y_{0}) \leftarrow (\emptyset,\ \cD)$.
            \FOR{$j=1, \dots, (i-1)$}
                \STATE $K_{j} \sim \cB(p_{j})$.
                \IF{$K_{j}$}
                    \STATE $(X_j,\ Y_j)\leftarrow (X_{j-1}\cup\{j\},\ Y_{j-1})$.
                \ELSE
                    \STATE $(X_j,\ Y_j)\leftarrow (X_{j-1},\ Y_{j-1}\setminus\{j\})$.
                \ENDIF 
            \ENDFOR
            \STATE \textbf{Return:} $(X_{i-1},\ Y_{i-1})$.
        \end{algorithmic}
    \end{algorithm}
\end{minipage}
\hfill
\begin{minipage}{0.52\linewidth}
    \begin{algorithm}[H]
    \caption{Exploration update (\UpdateExplo)}
    \begin{algorithmic}[1]
     \label{alg:UpdateExplo}
        \STATE \textbf{Inputs:} $i\in[d],\ (\alpha,\ \beta)\in[-c, c]^2,\ \tau\in\bN^*$.
        \STATE $\Lambda\leftarrow\{
            x\in[0, 1]\text{  s.t.  }\ell(\alpha,\ \beta,\ x) + \frac{g_{T, \delta}}{\sqrt{\tau}}\leq 0\}$.
        \STATE $p\leftarrow1/2$.
        \IF{$\Lambda \neq \emptyset$}
            \STATE $p\leftarrow \arg\min_{x\in\Lambda}\ell(\alpha, \beta, x)$.
            \STATE $i\leftarrow i + 1$.
            \ELSE
            \IF{$\tau \geq \tau_{\max }$}
                \STATE $p\leftarrow \frac{\alpha_{+}}{\alpha_{+} + \beta_{+}}$ where $(\cdot)_+ = \max\{\cdot,0\}$.
                \STATE $i\leftarrow i+1$.
            \ENDIF
        \ENDIF
        \STATE \textbf{Return: }$(p,\ i)$.
    \end{algorithmic}
    \end{algorithm}    
\end{minipage}
\end{center}

\subsection{Exploring just enough for zero exploitation regret: the key idea}

In \DGETC, the number of rounds devoted to the exploration for each item is adaptive, and is controlled by Subroutine \UpdateExplo. Given estimated marginal gains $(\hat \alpha_i,\ \hat{\beta}_i)$ and an exploration time $\tau$\,, \UpdateExplo~checks if it is possible to counterbalance the (high-probability) errors coming from the different sources of uncertainties.

On the one hand, the per-round exploitation regret induced by all sources of uncertainty (estimations errors, random sampling, noise fluctuations) for item $i$\,, is bounded with high-probability (Proposition~\ref{prop:concentrationExploration} in our analysis) by $\frac{g_{T, \delta}}{\sqrt{\tau_i}}$ where
\begin{align}
    \label{eq:gTdelta}\textstyle
    g_{T, \delta} = \sqrt{2(2\sigma^2+c^2)}\sqrt{2\log(dT)+\log(1/\delta)}\Bigg(1+2\sqrt{\frac{\log(dT)}{T}}+\frac{9c}{\sqrt{2\sigma^2+c^2}}\Big(\frac{\log(dT)}{T}\Big)^{1/3}\Bigg)\,.
\end{align}

On the other hand, the decision to either add or remove item $i$ with probability $p_i$ [Line 30 in \DGETC (Alg.~\ref{alg:DGETC}) and Lines 4-9 in \SamplingSet (Alg.~\ref{alg:SamplingSet})] induces an average loss\footnote{Average with respect to the sampling variable $K_i$.} per exploration round bounded by $\ell(\hat{\alpha}_i,\ \hat\beta_i,\ p_i)$ where
\begin{gather}
\label{eq:loss}
        \ell({\alpha},\ {\beta},\ p) = \max\big(\ell^+({\alpha},\ {\beta},\ p),\ \ell^-({\alpha},\ {\beta},\ p)\big)\,,\\
\textstyle\text{with }\,        \ell^+({\alpha},\ {\beta},\ p) = \big(1-p\big)\alpha - \frac{1}{2}(p\alpha+(1-p)\beta),\quad
        \ell^-({\alpha},\ {\beta},\ p) = p\beta - \frac{1}{2}(p\alpha+(1-p)\beta)\,.\nonumber
\end{gather}
In this definition, $\ell^+$ and $\ell^-$ are per-round regrets of using parameter $p$ when the (estimated) marginal gains are $(\alpha,\ \beta)$, corresponding to the two cases $\{i\in A^*\}$ and $\{i\notin A^*\}$. As one wants to hedge against both eventualities, we consider the worst-case loss $\ell$ which explains the max of both $\ell^+$ and $\ell^-$ in in Eq.~\eqref{eq:maxRT}.

\UpdateExplo~checks if, given estimations $(\hat \alpha_i,\ \hat \beta_i)$ and a current number of exploration rounds $4\tau_i$, it is possible to find a parameter $p_i$ so that the errors from uncertainties $\smash{\frac{g_{T, \delta}}{\sqrt{\tau_i}}}$ are absorbed by the (hopefully negative) loss $\ell(\hat{\alpha}_i,\ \hat{\beta}_i,\ p_i)$\,. Formally, it looks for the existence of a $p_i\in[0, 1]$ so that
\begin{align}
    \textstyle l(\hat \alpha_i,\ \hat\beta_i,\ p_i)+\frac{g_{T, \delta}}{\sqrt{\tau_i}}\leq 0\,,
\end{align}
which is guaranteed to happen after a logarithmic number of rounds (Proposition~\ref{prop:explorationSuff}). If it is the case, \UpdateExplo~returns this parameter $p_i$ and makes \DGETC move on to the next item. Otherwise, the exploration for the current item $i$ continues unless it has already lasted too long (i.e. if $\tau_i\geq\tau_{\max }$). In this case, \UpdateExplo~returns parameter $p_i=\frac{\hat \alpha_{i,+}}{\hat \alpha_{i,+} + \hat \beta_{i,+}}$ and makes \DGETC move on to the next step. This last choice for $p_i$ ensures the loss $\smash{\ell(\hat{\alpha_i},\ \hat{\beta}_i,\ p_i)}$ to be negative (or null) in the exploitation phase and the per-round regret for item $i$ to be bounded simply by $\frac{g_{T, \delta}}{\sqrt{\tau_{\max }}}$\,.

While \RGL~\citep{fourati2023randomizedgreedylearningnonmonotone} devotes the same number of rounds to all the items in the exploration phase, Subroutine \UpdateExplo~enables more flexibility. In particular, Section~\ref{sec:theory} links the number of exploration rounds necessary with problem-dependent quantities.

\begin{rremark}
    The possibility to counterbalance the accumulated errors with negative losses is enabled by the approximate regret criterion using the worst-case $1/2$ ratio, and an in-depth analysis of the original Double-Greedy algorithm. In all generality, this kind of intuition could also be applied to other methods to recover similar logarithmic upper bounds.
\end{rremark}

\section{Theoretical guarantees for \DGETC}
\label{sec:theory}

This section presents theoretical guarantees satisfied by our approach. We introduce a concept of problem-dependent \emph{hardness} that characterizes how difficult it can be to maximize a given submodular function with our \emph{Double-Greedy} approach. We then show that \DGETC satisfies logarithmic $1/2$-approximate pseudo-regret upper bounds which depend on this hardness, with a $O(dT^{2/3}\log(dT)^{1/3})$ worst-case.

\begin{rremark}
    We remind that the items are taken in an arbitrary order, and the quantities may depend on it.
\end{rremark}

\subsection{Double-Greedy hardness}
\label{sec:hardness}

The following hardness notion relates to the sufficient number of exploration rounds that guarantees to find parameters $(p_i)_{i\in[d]}$ suitable to induce zero $1/2$-approximate regret during the exploitation.

\begin{ddefinition}[DG-hardness] \label{def:DGhardness}
    Let $\cD$ be a set of $d$ elements (considered in an given order). Let $f$ be a submodular set-function over $\cD$ and $i$ be an item in $\cD$\,.
    
    We define the \emph{local DG-hardness} for item $i$ as\\
\begin{center}
\begin{minipage}{0.58\textwidth} 
\vspace{-2cm}
    \begin{align*}
        h_{f,i} = \max_{X\subseteq[i-1]} \frac{\big(\alpha_f(i, X)_+ +\beta_f(i, X)_+\big)^2}{\big(\alpha_f(i, X)_+-\beta_f(i, X)_+\big)^4}\,,
    \end{align*}
    where $(\,\cdot\,)_+ = \max\{\,\cdot\,,0\}$ and
    \begin{align*}
        &\alpha_f(i, X) = f(X\cup\{i\})-f(X)\,,\\
        &\beta_f(i, X) = f\big((\cD\setminus[i])\cup X\big)-f\big((\cD\setminus[i-1])\cup X\big)\,.
    \end{align*}
\noindent
We define the \emph{global DG-hardness} as $    H_f = \sum_{i\in[d]}h_{f,i}
$\,.
\end{minipage}
\hfill
\begin{minipage}{.4\textwidth}
\vspace{-1.5cm}
\includegraphics[width=\textwidth]{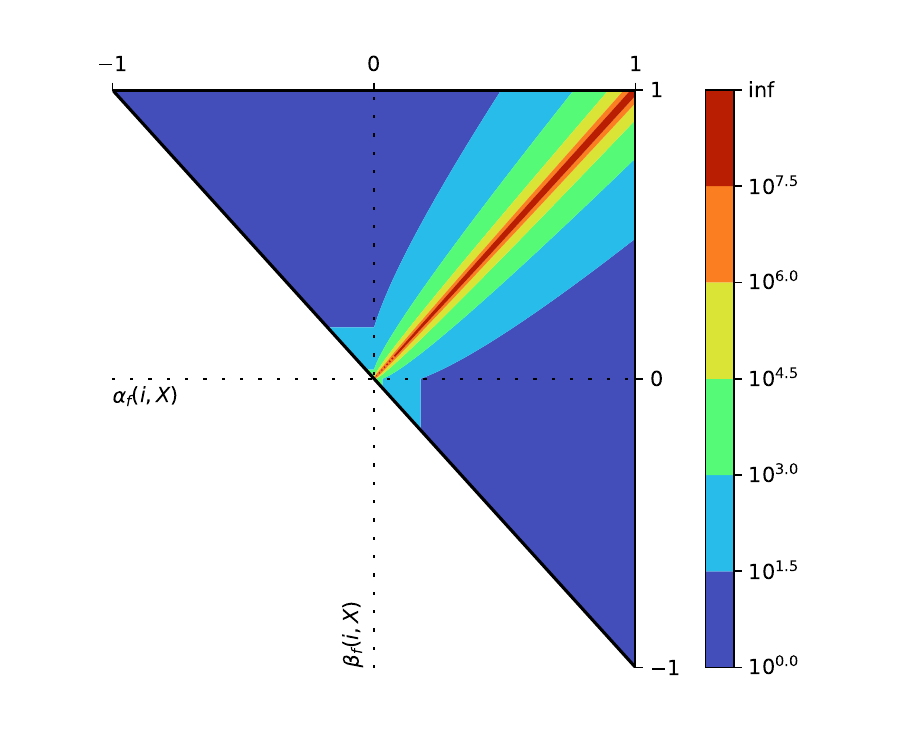}
\captionof{figure}{$h_{f,i}$ as a function of $\alpha_f(i,X)$ and $\beta_f(i,X)$ for $c=1$\,.}
\end{minipage}
\end{center}

\end{ddefinition}

\begin{rremark}
\begin{itemize}[nosep]
    \item This definition is actually not completely tight, as we will see in the analysis. But this form is more readable and convenient to use than the whole case disjunctions depending on the $(\alpha,\ \beta)$ configurations.
    \item We can also define a dual quantity, a \emph{local DG-gap} $\Delta_{f, i} = \big(h_{f,i}\big)^{-1/2}$\,, playing the same role as the suboptimality gaps in pseudo-regret upper bounds for stochastic multi-armed bandits (it is homogeneous to a difference of rewards). The corresponding \emph{global DG-gap} would be $\Delta_f = H_f^{-1/2}$\,.
\end{itemize}
\end{rremark}

\begin{eexample*} 
For illustration purposes, let's consider the following function $g$\,: we assume there exists $(\xi_i) \in [-1,1]^d$ and $\nu \in (0, 1]$ such that for all $X \subseteq [d]$\,,
\begin{equation}
\label{eq:submodular-example}
    g(X) = \Bigg(\sum_{i\in X,\ \xi_i\geq0} \xi_i\Bigg)^\nu - \Bigg(\sum_{i\in X,\ \xi_i<0} -\xi_i\Bigg)^{1 / \nu} + \|\xi_{-}\|_1^{1/\nu}\,,
\end{equation}
where $\xi_- = (\xi_i \mathds{1}\{\xi_i < 0\})_i$ and $\|\xi_{-}\|_1^{1/\nu}$ is here to guarantee the positivity of $g$\,.
Then $g$ is submodular and for all $i \in [d]$\,:
\begin{align*}
\Delta_{g,i} =\begin{cases}
g([i]) - g([i-1]) &\text{ if } \xi_i \geq 0\,, \\
g(\cD \setminus [i]) - g(\cD \setminus [i-1]) &\text{ if } \xi_i <0\,.
\end{cases}
\end{align*}
These expressions remind the notion of suboptimality gaps common in  bandit literature. If $\xi_i \geq 0$ then $i\in \OPT$ and the DG-gap corresponds to the reward gained by adding $i$ to $[i-1]$\,. If $\xi_i<0$ then $i \notin \OPT$ and the DG-gap corresponds to the reward increase when removing $i$ from $\{i, i+1, \dots, d\}$\,. 

Notably, when $g$ is linear ($\nu=1$), then the gaps $\Delta_{g,i} = \xi_i$ are independent from the ordering. They are intuitive as they correspond to the value of adding or removing the item, and the optimal set is the one containing all the items for which $\xi_i>0$ (assuming that no item has a gap of $0$).
\end{eexample*}

\subsection{Regret upper bounds for \DGETC}
\label{sec:regret}

This section presents our main result. We state a $1/2$-approximate pseudo-regret upper bounds for \DGETC, the proof of which is outlined in Section~\ref{sec:analysis}.

\begin{ttheorem}
\label{th:DGETC}
    Let $\cD$ be a finite set of $d\in\bN^*$ items, $T\in\bN^*$ a horizon with $d(T\sqrt{\log(dT)})^{2/3}\leq \frac{T}{2}$, $\sigma\in\bR^*_+$ and $c\in\bR^*_+$\,. Let $\delta>0$\,.
    
    Then, with probability greater than $1-10\delta/T$, \DGETC satisfies
    \[
        \smash{R_T
        \leq C_1\min\Big\{H_f\log(dT),\ dT^{2/3}\log(dT)^{1/3}\Big\}\,,}
    \]
    where $C_1$ is a constant independent from $d$, $T$ and $\delta$.
    
    Likewise, in expectation,
    \[
        \smash{\bE[R_T] \leq C_2 \min\Big\{H_f\log(dT),\ dT^{2/3}\log(dT)^{1/3}\Big\}\,,}
    \]
    where $C_1$ is a constant independent from $d$ and $T$\,.
\end{ttheorem}

\begin{rremark}
We can get more fine-grained bounds by using the local DG-hardnesses instead of the global one. From Eq.\ref{eq:template} at the end of Section~\ref{sec:analysis}, we can keep the per-item granularity to get with probability at least $1-10\delta/T$
\begin{align*}
    R_T\leq C_3 \sum_{i\in[d]}\min\Big\{h_{f,i}\log(dT),\ T^{2/3}\log(dT)^{1/3}\Big\}\,,
\end{align*}
where $C_3$ is a constant independent from $d$, $T$ and $\delta$\,. 

In particular, depending on the scale of the horizon $T$ with respect to the different local hardnesses $(h_{f, i})_{i\in[d]}$\,, we obtain a mixed sum of some logarithmic terms, and others of magnitude $T^{2/3}\log(dT)^{1/3}$\,.
\end{rremark}

\section{Analysis of \DGETC}
\label{sec:analysis}

This section presents a sketch of proof for Theorem~\ref{th:DGETC}.

We denote $\tau=\sum_{i\in[d]}4\tau_i$ the last exploration round. For all the items $i$\,, we also denote $t_i = \sum_{j\leq i}4\tau_i$\,, the last exploration round for item $i$\,. In this section, we have $i$-indices to denote items, and $t$-indices to denote that we place ourselves at round $t\in\bN^*$\,. When $t$ is not made explicit (notably for $\hat{\alpha}_i$\,, $\hat\beta_i$ and $p_i$), it means that we place ourselves after round $t_{i}$ (and that these parameters are fixed).

\paragraph{Outline of the proof. }
The idea of the proof is to find a high-probability event (namely, $\cE$) under which the exploration phase takes a logarithmic number of rounds per-item, and the regret is non positive during the exploitation phase. 
To that end, we first break the per-round regret of the exploitation phase down into per-item contributions (Section~\ref{sec:itemBreak}). Using this decomposition, we highlight an event $\cE$ under which the per-round, per-item, regret is bounded by $\smash{l(\hat\alpha_i,\ \hat\beta_i,\ p_i)+\frac{g_{T, \delta}}{\sqrt{\tau_i}}}$ for all the items $i$ (Section~\ref{sec:highProbaExploitation}). Lastly, we prove that under $\cE$\,, depending on the \emph{DG-hardness} of $f$ (Definition~\ref{def:DGhardness}), a logarithmic number of exploration rounds is sufficient to find a weights $p_i$ so that $\smash{l(\hat\alpha_i, \hat\beta_i, p_i)+\frac{g_{T, \delta}}{\sqrt{\tau_i}}\leq0}$ for all items $i$\,. Additionally, Subroutine \UpdateExplo~(Algorithm~\ref{alg:UpdateExplo}) returns a
parameters $p_i$ so that $l(\hat\alpha_i, \hat\beta_i, p_i)\leq 0$ when $\tau_i$ reaches $\tau_{\max }$ for item $i$ (Section~\ref{sec:explorationBound}), so that the regret for this item remains bounded by $T^{2/3}\log(dT)^{1/3}$ when the estimation needs more rounds than what can 
be afforded.

\paragraph{Template bound. }
Let $\cE$ be an event, defined later in Section \ref{sec:highProbaExploitation}. Then, the $1/2$-approximate pseudo-regret can be bounded as
\begin{align}
    \label{eq:template}
    R_T 
    &\leq \indicator_{\{\cE^c\}}\frac{cT}{2} + \indicator_{\{\cE\}}\Bigg( 2c\sum_{i=1}^d \tau_i + \frac{1}{2}\sum_{t=\tau+1}^Tr_t\Bigg)\,.
\end{align}
where $r_t = f(A^*)-2f(A_t)$.  

Under event $\cE^c$\,, the pseudo-regret is upper bounded by a worst case $cT/2$\,. Under $\cE$, each item $i$ uses $4\tau_i$ exploration rounds, each being bounded by a worst case $c/2$ regret, the rest of the rounds (between $\tau+1$ and $T$) are devoted to the exploitation and their regret is upper bounded in the following. In particular, they yield no regret when the exploration is successful and the instance is “easy" enough.

\subsection{Double-Greedy breakdown: Per-item exploitation regrets}
\label{sec:itemBreak}

We use an approach similar to \cite{Buchbinder2012} to bound the per-round exploitation regret $r_t$ with a sum of per-item contributions.

\paragraph{Item-wise breakdown.}
Let $t>\tau$\,. We considers sets $(A^*_{i, t})_{i\in[d]}$\,, with $A^*_{0, t}=A^*$ and $A^*_{d, t}=A_t$\,, constructed to control the evolution of $(f(A_{i, t}))_{i\in[d]}$ from $f(A^*)$ to $f(A_t)$ using the coefficients $(\alpha_{i, t},\ \beta_{i, t})_{i\in[d]}$\,. We define
\begin{align}
    &\text{For }i=0\,, &A^*_{0, t}=A^*\,,\hspace{1cm} &\text{with } X_{0, t} = \emptyset\,, \hspace{.5cm} Y_{0, t}=\cD\,,\nonumber\\
    &\forall i\in[d]\,, &  A^*_{i, t} = (A^*\cup X_{i, t}) \cap Y_{i, t}\,,\hspace{1cm} &\text{with } X_{i, t} \subseteq A^*_{i, t}\subseteq Y_{i, t}\,,\\
    &\text{For }i=d, &A^*_{d, t}=X_{d, t}=Y_{d, t}=A_t\,,\hspace{1cm} &\nonumber
\end{align}
where $X_{i, t}=\{j\leq i,\ K_{j, t}=1\}$ and $Y_{i, t}=\cD\setminus\{j\leq i,\ K_{j, t}=0\}$ are the sets defined in Subroutine \SamplingSet~(Algorithm~\ref{alg:SamplingSet}).

Using these sets and the definition of $r_t$ in Eq.~\eqref{eq:template}, a telescopic argument yields 
\begin{align}
\label{eq:telescopBreak}
    r_t 
    &\leq f(A^*)-f(A_t) - \frac{1}{2}\Big[2f(A_t) - (f(\emptyset)+f(\cD))\Big] \nonumber \hspace{3cm} \leftarrow(f\geq0)\\
    &=\Big[f(A^*_{0,t})-f(A^*_{d, t})\Big]-\frac{1}{2}\Big[f(X_{d,t})-f(X_{0, t})+f(Y_{d,t})-f(Y_{0, t})\Big] \nonumber\\
    &= \sum_{i=1}^d\Big[f(\OPT_{i-1, t})-f(\OPT_{i, t}) - \frac{1}{2}\big(K_{i, t}\alpha_{i, t}+(1-K_{i, t})\beta_{i, t}\big) \Big]\,,
\end{align}
where for all $i\in[d]$\,, $\begin{aligned}[t]
    \alpha_{i, t} &= f(X_{i-1, t}\cup \{i\}) - f(X_{i-1, t})\,,\\
    \beta_{i, t} &= f(Y_{i-1, t}\setminus \{i\}) - f(Y_{i-1, t})\,.
\end{aligned}$

\paragraph{Submodularity.}
While the marginal gains $(\alpha_{i, t},\ \beta_{i, t})_{i\in[d]}$ can be estimated, the sets $A^*$, and $(A^*_{i, t})_{i\in[d]}$ remain unknown. However, the definition of $(A^*_{i, t})_{i\in[d]}$ and submodularity yield
\begin{itemize}[nosep]
    \item If $\{i\in A^*\}$\,, then $
        f(\OPT_{i-1, t})-f(\OPT_{i, t}) \leq (1-K_{i, t})\alpha_{i, t}$\,,
    \item  Else $\{i\notin A^*\}$\,, and $f(\OPT_{i-1, t})-f(\OPT_{i, t}) \leq K_{i, t}\beta_{i, t}$\,.
\end{itemize}
Using these inequalities, Eq.~(\ref{eq:telescopBreak}) becomes
\begin{align*}
    r_t 
    &\leq \sum_{i\in[d]}\Big[\indicator_{\{i\in\OPT\}} (1-K_{i, t})\alpha_{i, t} + \indicator_{\{i\notin\OPT\}} K_{i, t}\beta_{i, t} - \frac{1}{2}\big(K_{i, t}\alpha_{i, t}+(1-K_{i, t})\beta_{i, t}\big)\Big]\,.
\end{align*}
Since $\{i\in\OPT\}$ and $\{i\notin\OPT\}$ are exclusive events, we have
\begin{align}
\label{eq:maxRT}
    \sum_{t=\tau+1}^Tr_t\leq \sum_{i\in[d]}\max\big\{R_{T, i}^+,\ R_{T, i}^-\big\}\,,
\end{align}
where $\begin{aligned}[t]
    R_{T, i}^+ &= \textstyle\sum_{t=\tau+1}^T \Big[(1-K_{i, t})\alpha_{i, t} -\frac{1}{2}\big(K_{i, t}\alpha_{i, t}+(1-K_{i, t})\beta_{i, t}\big)\Big]\,,\\
    R_{T, i}^- &= \textstyle\sum_{t=\tau+1}^T \Big[K_{i, t}\beta_{i, t} -\frac{1}{2}\big(K_{i, t}\alpha_{i, t}+(1-K_{i, t})\beta_{i, t}\big)\Big]\,.
\end{aligned}$

\subsection{High-probability exploitation regret}
\label{sec:highProbaExploitation}

Let $i\in[d]$, the objective now is to control $\max\big\{R_{T, i}^+,\ R_{T, i}^-\big\}$ from Eq.~\eqref{eq:maxRT}. To that end, the following proposition (proven in Appendix~\ref{app:highProbaExploitation}) states how the errors coming from the different randomness sources concentrate.

\begin{restatable}{pproposition}{concentrationExploration}
\label{prop:concentrationExploration}
    Let $\cH$ and $\cE$ be the event
    \begin{align*}
    \cH = \left.
    \begin{cases}\forall i\in[d]\,,\ \forall t> t_{i-1}\,,\
        &|\bar \alpha_i - \hat{\alpha}_{i, t}|\leq \sqrt{2\sigma^2+c^2}\sqrt{2\frac{\log(dT/\delta)+\log(1+4\tau_{i, t})}{\tau_{i, t}+1}}\,,\\
    &|\bar \beta_i - \hat{\beta}_{i, t}| \leq \sqrt{2\sigma^2+c^2}\sqrt{2\frac{\log(dT/\delta)+\log(1+4\tau_{i, t})}{\tau_{i, t}+1}}
    \end{cases}\right\}\,,\\
    \cE = \cH\cap\Bigg\{\forall i\in[d]\,,\hspace{.5cm}\max\big\{R_{T, i}^+, R_{T, i}^-\big\} -(T-\tau)\Big(l(\hat\alpha_i,\ \hat\beta_i,\ p_i) + \frac{g_{T, \delta}}{\sqrt{\tau_i}}\Big)\leq 0\Bigg\}\,,
\end{align*}
where for all $i\in[d]$\,, $\bar\alpha_{i} = \bE\big[\alpha_{i, t}|(p_j)_{j < i}\big]$ and $\bar \beta_i = \bE\big[\beta_{i, t}|(p_j)_{j < i}\big]$\,, both quantities being constant for rounds $t> t_{i-1}$\,, and $g_{T, \delta}$ is defined in Eq.~\eqref{eq:gTdelta}.

Then, $\bP(\cH^c)\leq \frac{4\delta}{T}$\,, and $\bP(\cE^c)\leq \frac{10\delta}{T}$\,.
\end{restatable}

\paragraph{Template bound. } Reinjecting Eq.~\eqref{eq:maxRT} and Proposition~\ref{prop:concentrationExploration} yields
\begin{align}
\label{eq:templateConc}
    R_T 
    &\leq \indicator_{\{\cE^c\}}\frac{cT}{2} + \indicator_{\{\cE\}}\sum_{i\in[d]}\Bigg( 2c\tau_i + (T-\tau) \Big(l(\hat\alpha_i, \hat\beta_i, p_i) + \frac{g_{T, \delta}}{\sqrt{\tau_i}}\Big)\Bigg)\,,
\end{align}
where $\cE$ is the event defined Proposition \ref{prop:concentrationExploration}.

\subsection{Sufficient exploration}
\label{sec:explorationBound}

In this section, we analyze the exploration steps for each item and we exhibit sufficient conditions for them to only last a logarithmic number of rounds. The default choice of $\smash{p_i=\frac{\alpha_{i,+}}{\alpha_{i,+} + \beta_{i,+}}}$ when $\tau_i\geq \tau_{\max }$ in Subroutine \UpdateExplo~(Algorithm \ref{alg:UpdateExplo}) ensures $\ell(\hat{\alpha}_i,\ \hat \beta_i,\ p_i)\leq0$\,, which in turns yield a $\smash{O\big(T\sqrt{\log(dT))^{2/3}}\big)}$ regret upper bound for item $i$.

Subroutine \UpdateExplo~looks for a parameter $p\in[0, 1]$ so that both
\begin{align}
    \label{eq:condHat}
    \big(1-p\big)\hat\alpha_i - \frac{1}{2}\big(p\hat \alpha_i+(1-p)\hat \beta_i\big) \leq -\frac{g_{T,\delta}}{\sqrt{\tau_i}}\,,\hspace{.3cm} \text{and}\hspace{.3cm}
    p\hat\beta_i - \frac{1}{2}(p\hat \alpha_i+(1-p)\hat \beta_i) \leq -\frac{g_{T,\delta}}{\sqrt{\tau_i}}\,.
\end{align}

\noindent
Under $\cE$\,, as we can upper bound $|\hat{\alpha}_i-\bar\alpha_i|$ and $|\hat{\beta}_i-\bar\beta_i|$ (Proposition~\ref{prop:concentrationExploration}), it is sufficient to have
\begin{align*}
    \begin{cases}
        \Big(1-p\Big)\bar\alpha_i - \frac{1}{2}(p\bar \alpha_i+(1-p)\bar \beta_i) &\leq -\frac{g_{T, \delta}}{\sqrt{\tau_i}}-\frac{3}{2}\sqrt{2\sigma^2+c^2}\sqrt{2\frac{\log(dT/\delta)+\log(1+4\tau_i)}{\tau_i+1}}\\
        p\bar\beta_i - \frac{1}{2}(p\bar \alpha_i+(1-p)\bar \beta_i) &\leq -\frac{g_{T, \delta}}{\sqrt{\tau_i}}-\frac{3}{2}\sqrt{2\sigma^2+c^2}\sqrt{2\frac{\log(dT/\delta)+\log(1+4\tau_i)}{\tau_i+1}}\,,
    \end{cases}
\end{align*}
for which it is in turn sufficient to have
\begin{align}
\label{eq:pConditions}
    p(\bar\beta_i-3\bar\alpha_i) \leq -\frac{g_i+\gamma_{T, \delta}}{\sqrt{\tau_i}}+(\beta_i-2\bar\alpha_i), \hspace{.3cm} \text{and}\hspace{.3cm}
    p(3\bar\beta_i - \bar \alpha_i) \leq -\frac{g_i+\gamma_{T, \delta}}{\sqrt{\tau_i}}+\bar\beta_i\,,
\end{align}
where $\gamma_{T, \delta}=3\sqrt{(2\sigma^2+c^2)(\log(dT/\delta)+\log(1+T))}$.

The following proposition gives sufficient conditions to find a $p_i$ for Eq.~\eqref{eq:condHat} to be satisfied.

\begin{restatable}{pproposition}{explorationSuff}
\label{prop:explorationSuff}
    For each items $i\in[d]$\,, under event $\cE$ defined in Proposition~\ref{prop:concentrationExploration}, \UpdateExplo~finds a weight $p_i$ such that $l(\hat \alpha_{i, t},\ \hat \beta_{i, t},\ p_i)+\frac{g_{T, \delta}}{\sqrt{\tau_{i, t}}}\leq 0$ before $4\tau_{i, t}$ the current number of exploration rounds for item $i$ has reached the value $4(g_{T,\delta}+\gamma_{T, \delta})^2\ h_{f,i}$\,.
\end{restatable}

\paragraph{Template bound. } Using Proposition~\ref{prop:explorationSuff}, the upper bound Eq.~\eqref{eq:templateConc} becomes
\begin{align}
    R_T 
    &\leq \indicator_{\{\cE^c\}}\frac{cT}{2} + \indicator_{\{\cE\}}\sum_{i\in[d]}\Bigg[2c\min\Big\{(g_{T,\delta}+\gamma_{T, \delta})^2h_{f,i}\,,\ \tau_{\max }\Big\}\nonumber\\
    &\hspace{5cm}+\indicator_{\big\{(g_{T,\delta}+\gamma_{T, \delta})^2h_{f,i}> \tau_{\max }\big\}} \tau_{\max }\frac{Tg_{T, \delta}}{(\tau_{\max })^{3/2}}\Bigg]\nonumber\\
    &=\indicator_{\{\cE^c\}}\frac{cT}{2} + \indicator_{\{\cE\}}\sum_{i\in[d]}\Big(2c + \frac{g_{T, \delta}}{\log(dT)^{1/2}}\Big)\min\Big\{(g_{T,\delta}+\gamma_{T, \delta})^2h_{f,i}\,,\ \tau_{\max }\Big\}\,,
    \label{eq:templateFinal}
\end{align}
where $\tau_{\max}= T^{2/3}\log(dT)^{1/3}$ is defined in Eq.\eqref{eq:tauMax}.

The high-probability result comes from event $\cE$ happening with probability greater than $1-\frac{10\delta}{T}$ (Proposition~\ref{prop:concentrationExploration}). Choosing $\delta=1$ yields the bound in expectation.

\section{Concluding remarks}

We propose and analyze Algorithm \DGETC (Algorithm~\ref{alg:DGETC}) for the online unconstrained submodular maximization problem, with stochastic bandit feedback. Our algorithm is a considerable improvement from other existing approaches, as it satisfies logarithmic upper bounds for the $1/2$-approximate pseudo-regret, dependent on a new notion of hardness that we introduce. Possible extensions include designing anytime variants, and algorithms adaptive to the adversarial/stochastic setting (best-of-both worlds). 

An interesting feature of \DGETC is that it leverages the looseness of worst-case approximation ratios in non-adversarial cases, and we argue that this kind of strategy could also be applied to other settings to yields similar performances.
 
% Acknowledgments---Will not appear in anonymized version
% \acks{We thank a bunch of people and funding agency.}

\newpage

\bibliography{references}

\newpage

\appendix

\crefalias{section}{appendix} % uncomment if you are using cleveref
\DoToC

\section{Extended related works}
\label{app:related}

This section completes the related work of the main paper in Section~\ref{sec:related}. 

We discuss additional submodular optimization settings existing in the literature, namely, offline and online minimization, constrained maximization, as well as other maximization frameworks.

\paragraph{Submodular optimization (offline).} Submodularity is also studied in settings different from the unconstrained non monotone maximization case that we look at.

\emph{Minimization} can be solved in polynomial time \citep{grotschel1981ellipsoid} and there is an extensive line of work studying that setting. See \cite{lee2015faster, chakrabarty2017subquadratic,axelrod2020near, jiang2021minimizing, jiang2022minimizing, jiang2023convex} for the most recent ones. 

For \emph{maximization}, the constrained non monotone setting is even more challenging and there is a line of work constantly improving the approximations \citep{lee2009non, chekuri2011submodular, buchbinder2014submodular, vondrak2013symmetry, ene2016constrained, buchbinder2019constrained, tukan2024practical, buchbinder2024constrained}, which is known to be smaller than $.478$ in polynomial time \citep{qi2024maximizing}. In the monotone setting, there also exist a line of works interested in the identification of the best subset with the minimal number of potentially noisy calls to the function \citep{singla2016noisy, hassidim2017submodular, karimi2017stochastic, hassidim2018optimization}. Other papers also study maximization of a submodular function only known from given samples \citep{balkanski2016power, balkanski2017limitations}.

\paragraph{Online and bandit submodular optimization.} Both maximization and minimization are explored in the online learning and bandit literature.

For the \emph{minimization} version, \citet{hazan2012online} introduces an online adversarial setting and proposes an algorithm with sublinear regret. Its results are improved in \citet{matsuoka2021tracking} and \citet{ito2022revisiting}, the latter also proposing results for the bandit setting. In particular, a commonly used tool is the Lovász extension, reducing the problem to convex minimization.

The \emph{maximization} version is studied in \cite{streeter2008online} with a resource allocation perspective, giving guarantees for the $(1-1/e)$-approximate expected regret in the monotone setting. Its results are extended to matroid constraints in \citet{streeter2009online, golovin2014online}, and improved in \citet{harvey2020improved} using curvatures. The case of bandit feedback is also studied for monotone functions. \citet{yue2011linear} and \citet{guillory2011online} are early works works providing theoretical guarantees. They are followed by several papers considering variants of this setting \citep{kohli2013fast, gabillon2013adaptive, chen2018contextual}. 

\newpage
\section{Reminders on sub-Gaussianity}
\label{app:concentration}

We use sub-Gaussianity assumptions and common concentration tools to control deviations of the noise $(\eta_t)_{t\in[T]}$ and the randomization process of our approach. This section remind useful results.

\begin{restatable}[Sub-Gaussian]{ddefinition}{defSubgaussian}
Let $\sigma>0$ and $X$ be a real-valued random variable such that $\bE[X]=0$. We say that $X$ is $\sigma^2$-sub-Gaussian,  for all $\lambda\in\bR$,
\begin{align*}
    \bE[\exp(\lambda X)] \leq \exp\Big(\frac{\lambda^2\sigma^2}{2}\Big)\,.
\end{align*}
\end{restatable}

In particular, for bounded independent random variables, we have the following lemma.

\begin{restatable}[Hoeffding's inequality for sum of i.i.d. bounded r.v.]{llemma}{lemIidSubgaussian}
\label{lem:lemIidSubgaussian}
Let $\delta>0$, $N\in\bN^*$, and $(Z_n)_{n\in[N]}$ a family of i.i.d. real random variables bounded in $[a, b]$ where $(a, b)\in(\bR)^2$, with mean $\mu\in[a, b]$.

Then for all $n\in[N]$, $Z_n$ is $\frac{(b-a)^2}{4}$-sub-Gaussian, and with probability at least $1-\delta$,
\begin{align*}
    \frac{1}{N}\sum_{n=1}^N\big[Z_n-\mu\big] < \frac{b-a}{2}\sqrt{\frac{2}{N}\log(1/\delta)}\,.
\end{align*}
\end{restatable}

The sub-Gaussianity for bounded random variables an the concentration for the sums of i.i.d random variables are classical results proven that can be found \cite{wainwright2019high} for example.

As we estimate quantities in an online setting, with observations arriving sequentially and depending on our actions, we need a more powerful tool. This is provided by the following lemma.

\begin{restatable}[Hoeffding's inequality with martingales]{llemma}{lemOnlineConcentration}
\label{lem:onlineConc}
    Let $\delta>0$, $\sigma>0$. Let $(\cG_t)_{t\in\bN}$ be a filtration and $(Z_t)_{t\in\bN^*}$ a $(\cG_t)$-adapted martingale with $\bE[Z_1]=0$. We assume that for all $t\in\bN$, $Z_{t+1}$ is $\sigma^2$-sub-Gaussian conditionally to $\cG_t$. Let $(U_t)_{t\in\bN^*}$ be a $(\cG_t)$-predictable process.     
    Then, with probability at least $1-\delta$, for all $t\in\bN$
    \begin{align*}
        \frac{\sum_{s=1}^t U_sZ_s}{1+\sum_{s=1}^tU_s^2}< \frac{\sigma}{\sqrt{1+\sum_{s=1}^tU_s^2}}\sqrt{2\log(1/\delta)+\log\Big(1+\sum_{s=1}^tU_s^2\Big)}
    \end{align*}
\end{restatable}

The proof relies on the 
 method of mixture, widely used in the bandit literature \citep{abbasi2011improved, faury2020improved, zhou2024efficientoptimalcovarianceadaptivealgorithms}.

\begin{proof}
    Let $\delta>0$, $\sigma>0$. Let $(\cG_t)$ be a filtration and $(Z_t)$ be a $\cG_t$-adapted martingale with $\bE[Z_1]=0$ and so that for all $t\in\bN$,  $Z_{t+1}$ is $\sigma^2$-sub-Gaussian conditionally to $\cG_t$. Let $(U_t)$ be a $\cG_t$-predictable process.

    Let $t\in\bN^*$, a first direct result is that, $U_{t}Z_{t}$ is $(\sigma U_{t})^2$-sub-Gaussian conditionally to $\cG_{t-1}$. Let $\lambda\in\bR$. Then,
    \begin{align}
        \label{eq:superMartStep}
        \bE\Big[\exp\Big(\lambda U_{t}Z_{t}-\frac{\lambda^2}{2}(\sigma U_{t})^2\Big) | \cG_{t-1}\Big]\leq 1\,.
    \end{align}
    We define
    \begin{align*}
        M_t(\lambda) = \exp\Bigg(\lambda \sum_{s=1}^tU_{s}Z_{s}-\frac{\lambda^2}{2}\sum_{s=1}^t(\sigma U_{s})^2\Bigg)\,
    \end{align*}
    with $M_0(\lambda)=1$.
    From \cref{eq:superMartStep},
    \begin{align*}
        \forall t\in\bN,\hspace{0.5cm} \bE[M_t(\lambda)|\cG_t] 
        &= \bE\Bigg[\exp\Bigg(\lambda \sum_{s=1}^tU_{s}Z_{s}-\frac{\lambda^2}{2}\sum_{s=1}^t(\sigma U_{s})^2\Bigg)\Bigg|\cG_{t-1}\Bigg]\\
        &= M_{t-1
        }(\lambda)\ \bE\Bigg[\exp\Bigg(\lambda U_{t}Z_{t}-\frac{\lambda^2}{2}(\sigma U_{t})^2\Bigg)\Bigg|\cG_{t-1}\Bigg]\\
        &\leq M_{t-1}(\lambda)\,.
    \end{align*}
    Then, $(M_t(\lambda))_t$ is a $\cG_t$-supermartingale, with $\bE[M_t(\lambda)]\leq 1$\,. 

    We now consider $\lambda\sim \cN(0, 1/\sigma^2)$, independent from all the other distributions, then we can define
    {\allowdisplaybreaks
    \begin{align*}
    \allowdisplaybreaks
        \bar M_{t} 
        &= \bE_{\lambda\sim\cN(0, 1/\sigma^2)}[M_t(\lambda)]\\
        &= \frac{\sigma}{\sqrt{2\pi}}\int_{\bR} \exp\Big(-\frac{(\sigma x)^2}{2}\Big) \exp\Big(x \sum_{s=1}^tU_{s}Z_{s}-\frac{x^2}{2}\sum_{s=1}^t(\sigma U_{s})^2\Big) dx\\
        &= \frac{\sigma}{\sqrt{2\pi}}\int_{\bR} \exp\Big(-\frac{(\sigma x)^2(1+\sum_{s=1}^tU_{s}^2)}{2} + x \sum_{s=1}^tU_{s}Z_{s}\Big) dx\\
        &= \frac{\sigma}{\sqrt{2\pi}}\int_{\bR} \exp\Bigg(-\frac{\sigma^2(1+\sum_{s=1}^tU_s^2)}{2}\Big(x^2-2x\frac{\sum_{s=1}^tU_sZ_s}{\sigma^2(1+\sum_{s=1}^tU_s^2)}\Big)\Bigg) dx\\
        &= \frac{\sigma}{\sqrt{2\pi}}\int_{\bR} \exp\Bigg(-\frac{\sigma^2(1+\sum_{s=1}^tU_s^2)}{2}\Big(x-\frac{\sum_{s=1}^tU_sZ_s}{\sigma^2(1+\sum_{s=1}^tU_s^2)}\Big)^2+\frac{(\sum_{s=1}^tU_sZ_s)^2}{2\sigma^2(1+\sum_{s=1}^tU_s^2)}\Bigg) dx\\
        &= \exp\Bigg(\frac{(\sum_{s=1}^tU_sZ_s)^2}{2\sigma^2(1+\sum_{s=1}^tU_s^2)}\Bigg)\frac{\sigma}{\sqrt{2\pi}}\ \frac{\sqrt{2\pi}}{\sigma\sqrt{1+\sum_{s=1}^tU_s^2}}\frac{\sigma\sqrt{1+\sum_{s=1}^tU_s^2}}{\sqrt{2\pi}}\\
        &\hspace{4cm}\int_{\bR} \exp\Bigg(-\frac{\sigma^2(1+\sum_{s=1}^tU_s^2)}{2}\Big(x-\frac{\sum_{s=1}^tU_sZ_s}{\sigma^2(1+\sum_{s=1}^tU_s^2)}\Big)^2\Bigg) dx\\
        &=\exp\Bigg(\frac{(\sum_{s=1}^tU_sZ_s)^2}{2\sigma^2(1+\sum_{s=1}^tU_s^2)}\Bigg)\frac{1}{\sqrt{1+\sum_{s=1}^tU_s^2}}\bE_{\lambda\sim\cN\big(\frac{\sum_{s=1}^tU_sZ_s}{\sigma^2(1+\sum_{s=1}^tU_s^2)},\ \frac{1}{\sigma^2(1+\sum_{s=1}^tU_t^2)}\big)}\Big[1\Big]\\
        &=\frac{1}{\sqrt{1+\sum_{s=1}^tU_s^2}}\exp\Bigg(\frac{(\sum_{s=1}^tU_sZ_s)^2}{2\sigma^2(1+\sum_{s=1}^tU_s^2)}\Bigg)\\
        \bar M_t &=\exp\Bigg(\frac{(\sum_{s=1}^tU_sZ_s)^2}{2\sigma^2(1+\sum_{s=1}^tU_s^2)} - \frac{1}{2}\log\Big(1+\sum_{s=1}^tU_s^2\Big)\Bigg)
    \end{align*}
    }
    Besides, 
    \begin{align*}
        \bE\Big[\bar M_t \Big| \cG_{t-1}\Big] 
        &= \bE\Big[\bE_{\lambda\sim \cN(0, 1/\sigma^2)}[M_t(\lambda)]\Big|\cG_{t-1}\Big]\\
        &= \bE_{\lambda\sim \cN(0, 1/\sigma^2)}\Big[\bE[M_t(\lambda)|\cG_{t-1}]\Big]\\
        &\leq \bE_{\lambda\sim \cN(0, 1/\sigma^2)}\Big[M_{t-1}(\lambda)\Big]\\
        &= \bar M_{t-1}\,.
    \end{align*}
    So $(\bar M_{t})_t$ is also a supermartingale, which yield that 
    \begin{align*}
        \bE[\bar M_{t}]\leq \bE[\bar M_0] = 1\,.
    \end{align*}
    Let $u_t>0$\,. Now, using Chernoff's method,
    \begin{align*}
        \bP\Bigg(\frac{\sum_{s=1}^tU_sZ_s}{1+\sum_{s=1}^tU_s^2}\geq u_t\Bigg)
        &\leq \bP\Bigg(\exp\Big(\frac{(\sum_{s=1}^tU_sZ_s)^2}{2\sigma^2(1+\sum_{s=1}^tU_s^2)}-\frac{u_t^2}{2\sigma^2}(1+\sum_{s=1}^tU_s^2)\Big)\geq 1\Bigg) \\
        &\leq \bE\Bigg[\exp\Big(\frac{(\sum_{s=1}^tU_sZ_s)^2}{2\sigma^2(1+\sum_{s=1}^tU_s^2)}-\frac{u_t^2}{2\sigma^2}(1+\sum_{s=1}^tU_s^2)\Big)\Bigg]\\
        &\leq \bE\Bigg[\bar M_t\exp\Big(\frac{1}{2}\log(1+\sum_{s=1}^tU_s^2)-\frac{u_t^2}{2\sigma^2}(1+\sum_{s=1}^tU_s^2)\Big)\Bigg]\,.
        \end{align*}
Choosing $u_t= \frac{\sigma}{\sqrt{1+\sum_{s=1}^t}U_s^2}\sqrt{2\log(1/\delta)+\log(1+\sum_{s=1}^tU_s^2)}$\,,
\begin{align*}
    \bP\Bigg(\frac{\sum_{s=1}^tU_sZ_s}{1+\sum_{s=1}^tU_s^2}\geq u_t\Bigg)
    &\leq \bE\Big[\delta \bar M_t\Big]\\
    &\leq\delta\,.
\end{align*}

The bound for all $t$ is based on the stopping time construction from \cite{abbasi2011improved}.

\end{proof}

\newpage
\section{Proof for the high-probability bound of Double-Greedy (Algorithm~\ref{alg:DG}, \DG from \citet{Buchbinder2012})}
\label{app:highProbaDG}

\highProbaDG*

\begin{proof}
    Let $1>\delta>0$ and $T\in\bN^*$ such that $T>2\log(1/\delta)$\,. Then $(f(S_i))_{i\in[T]}$ is a sequence of $T$ i.i.d. random variables, bounded in $[0, f(\OPT)]$.

    Let $\frac{1}{2}>u>0$. Then
    \begin{align*}
        \bP\Bigg(\max_{i\in[T]}f(S_i) < \Big(\frac{1}{2}-u\Big) f(\OPT)\Bigg) 
        &= \bP\Bigg(\forall i\in[T],\ f(S_i) < \Big(\frac{1}{2}-u\Big) f(\OPT)\Bigg)\\
        &= \prod_{i=1}^T\bP\Bigg(f(S_i) < \Big(\frac{1}{2}-u\Big)f(\OPT)\Bigg)\\
        &\leq \bP\Bigg(\Big(\frac{1}{2}-u\Big)f(\OPT) + f(\OPT)-f(S_1)> f(\OPT)\Bigg)^T\\
        &\leq \frac{1}{f(\OPT)^T}\bE\Bigg[\Big(\frac{1}{2}-u\Big)f(\OPT)+f(\OPT)-f(S_1)\Bigg]^T \leftarrow \text{Markov}\\
        &\leq \frac{1}{f(\OPT)^T}\Bigg[\Big(\frac{1}{2}-u\Big)f(\OPT)+\frac{1}{2}f(\OPT)\Bigg]^T \leftarrow\text{Theorem~\ref{th:DG}}\\
        &=(1-u)^T\\
        &\leq \exp(-Tu)\,.
    \end{align*}
    Therefore, taking $u=\frac{\log(1/\delta)}{T}$, we have the result
    \begin{align*}
        \bP\Bigg(\max_{i\in[T]}f(S_i) < \Big(\frac{1}{2}-u\Big) f(\OPT)\Bigg) \leq \delta\,.
    \end{align*}
\end{proof}

\newpage
\section{Proofs for the analysis of Double-Greedy - Explore-Then-Commit (\DGETC, ours)}

\subsection{Proof for the high-probability exploitation regret}
\label{app:highProbaExploitation}

\concentrationExploration*

\begin{proof}
We remind Eq.~\eqref{eq:maxRT},
\begin{align}
    \textstyle\sum_{t=\tau+1}^Tr_t\leq \textstyle\sum_{i\in[d]}\max\big\{R_{T, i}^+,\ R_{T, i}^-\big\}\,, \tag{\ref{eq:maxRT}}
\end{align}
where $\begin{aligned}[t]
    R_{T, i}^+ &= \textstyle\sum_{t=\tau+1}^T \Big[(1-K_{i, t})\alpha_{i, t} -\frac{1}{2}\big(K_{i, t}\alpha_{i, t}+(1-K_{i, t})\beta_{i, t}\big)\Big]\,,\\
    R_{T, i}^- &= \textstyle\sum_{t=\tau+1}^T \Big[K_{i, t}\beta_{i, t} -\frac{1}{2}\big(K_{i, t}\alpha_{i, t}+(1-K_{i, t})\beta_{i, t}\big)\Big]\,.
\end{aligned}$

For $i\in[d]$, we define $\bar\alpha_{i} = \bE\big[\alpha_{i, t}|(p_j)_{j < i}\big]$ and $\bar \beta_i = \bE\big[\beta_{i, t}|(p_j)_{j < i}\big]$\,, both quantities being constant for rounds $t> t_{i-1}$ (and thus, for $t> \tau$). Separating the different sources of randomness yields
\begin{align*}
    R_{T, i}^+ = \bar E_{T, i}^+ + \hat{E}_{T, i}^+ + L_{T, i}^+\,,\hspace{1cm} R_{T, i}^-= \bar E_{T, i}^- + \hat{E}_{T, i}^- + L_{T, i}^-\,,
\end{align*}
where we have
\begin{itemize}[nosep]
    \item errors coming from the deviation of $(\alpha_{i,t},\ \beta_{i,t})$ from $(\bar\alpha_i,\ \bar\beta_i)$\,:
    \begin{itemize}[nosep]
        \item[] $\bar E_{T, i}^+ = \sum_{t=\tau+1}^T\Big[(1-K_{i, t})(\alpha_{i, t}-\bar\alpha_i)-\frac{1}{2}\big(K_{i, t}(\alpha_{i, t}-\bar\alpha_i)+K_{i, t}^c(\beta_{i, t}-\bar\beta_i)\big)\Big]$,
        \item[] $\bar E_{T, i}^- = \sum_{t=\tau+1}^T\Big[K_{i, t}(\beta_{i, t}-\bar\beta_i)-\frac{1}{2}\big(K_{i, t}(\alpha_{i, t}-\bar\alpha_i)+K_{i, t}^c(\beta_{i, t}-\bar\beta_i)\big)\Big]$,
    \end{itemize}
    \item approximation errors for $(\hat{\alpha}_i, \hat{\beta}_i)$\,:
    \begin{itemize}[nosep]
        \item[] $\hat E_{T, i}^+ = \sum_{t=\tau+1}^T\Big[(1-K_{i, t})(\bar\alpha_{i, t}-\hat\alpha_i) -\frac{1}{2}\big(K_{i, t}(\bar\alpha_{i, t}-\hat\alpha_i)+(1-K_{i, t})(\bar\beta_{i, t}-\hat\beta_i)\big)\Big]$,
        \item[] $\hat E_{T, i}^- = \sum_{t=\tau+1}^T\Big[K_{i, t}(\bar\beta_{i, t}-\hat\beta_i) -\frac{1}{2}\big(K_{i, t}(\bar\alpha_{i, t}-\hat\alpha_i)+(1-K_{i, t})(\bar\beta_{i, t}-\hat\beta_i)\big)\Big]$,
    \end{itemize}
    \item the deviation of losses caused by the random variables $(K_{i,t})_{i,t}$\,:
    \begin{itemize}[nosep]
        \item[] $L_{T, i}^+=\sum_{t=\tau+1}^T\Big[(1-K_{i, t})\hat{\alpha}_{i} - \frac{1}{2}(K_{i, t}\hat \alpha_{i}+(1-K_{i, t})\hat \beta_{i})\Big]-(T-\tau)l_{i}^+$,
        \item[] $L_{T, i}^-=\sum_{t=\tau+1}^T\Big[(K_{i})\hat{\beta}_{i} - \frac{1}{2}(K_{i}, t\hat \alpha_{i}+(1-K_{i, t})\hat \beta_{i})\Big]-(T-\tau)l_i^-$,
    \end{itemize}
    \item the average loss criterion used in \UpdateExplo\,:
    \begin{itemize}[nosep]
        \item[] $l_{i}^+ = l^+(\hat{\alpha}_i,\ \hat{\beta}_i,\ p_i)=(1-p_{i})\hat \alpha_{i} - \frac{1}{2}(p_{i}\hat \alpha_{i} + (1-p_{i})\hat \alpha_{i})$,
    \item[] $l_{i}^- = l^-(\hat{\alpha}_i,\ \hat{\beta}_i,\ p_i)=p_{i}\hat \beta_{i} - \frac{1}{2}(p_{i}\hat \alpha_{i} + (1-p_{i})\hat \alpha_{i})$.
    \end{itemize}
\end{itemize}

We control these terms using the concentration lemmas in Appendix~\ref{app:concentration}. Let $\delta>0$, we define 
\begin{align*}
    \cG &= \Big\{\forall i,\ \bar E_{T, i}^+ \leq3c \sqrt{2(T-\tau)\log(dT/\delta)} \text{ and } \bar E_{T, i}^- \leq3c \sqrt{2(T-\tau)\log(dT/\delta)}\Big\}\,,\\
    \cH &=\left.
    \begin{cases}\forall i\in[d],\ \forall t> t_{i-1},\
        &|\bar \alpha_i - \hat{\alpha}_{i, t}|\leq \sqrt{2\sigma^2+c^2}\sqrt{2\frac{\log(dT/\delta)+\log(1+4\tau_{i, t})}{\tau_{i, t}+1}}\,;\\
    &|\bar \beta_i - \hat{\beta}_{i, t}| \leq \sqrt{2\sigma^2+c^2}\sqrt{2\frac{\log(dT/\delta)+\log(1+4\tau_{i, t})}{\tau_{i, t}+1}}
    \end{cases}\right\}\,,\\
    \cI &= \left.\begin{cases}
        \forall i\in[d],\ 
        &\hat E_{T, i}^+ \leq (T-\tau)\Big(1+\frac{\sqrt{2\log(dT/\delta)}}{\sqrt{T-\tau}}\Big)\max\Big(|\bar\alpha_i-\hat\alpha_{ i}|,|\bar\beta_i-\hat\beta_{i}|\Big)\\
        &\hat E_{T, i}^- \leq (T-\tau)\Big(1+\frac{\sqrt{2\log(dT/\delta)}}{\sqrt{T-\tau}}\Big)\max\Big(|\bar\alpha_i-\hat\alpha_{i}|,|\bar\beta_i-\hat\beta_{t}|\Big)
    \end{cases}
    \right\}\,,\\
    \cJ &= \left.\begin{cases}
        \forall i\in[d],\ &L_{T, i}^+\leq \frac{3c}{\sqrt{2}}\sqrt{(T-\tau)\log(dT/\delta)}\,,\\
        &L_{T, i}^-\leq \frac{3c}{\sqrt{2}}\sqrt{(T-\tau)\log(dT/\delta)}
    \end{cases}\right\}\,.
\end{align*}

Applying Lemma~\ref{lem:lemIidSubgaussian} and a union bound yields that $\bP(\cG^c\cup\cI^c\cup\cJ^c) \leq \frac{6\delta}{T}$. Likewise Lemma~\ref{lem:onlineConc} yields $\bP(\cH^c)\leq \frac{4\delta}{T}$.

Besides , $\cG\cap\cH\cap\cI\cap\cJ\subseteq\cE$ (calculations assuming $d\geq 2$, $\tau\leq T/2$, and $dT\geq\delta$).
\end{proof}

\subsection{Proof for the duration of the exploration phase}
\label{app:exploration}

The following lemma is a consequence of the definition of submodularity, but it is particularly useful when analyzing double-greedy approaches, as it limits the range of possible marginal gains to consider when adding/removing items.

\begin{llemma}
\label{lem:marginal_sum}
    Let $\cD$ be a finite set and $f$ be a submodular set-function. Let $A\subset B\subseteq\cD$ and an item $i\in(B\setminus A)$\,. Then, 
    \begin{align*}
        \Big(f(A\cup\{i\}) - f(A)\Big) + \Big(f(B\setminus\{i\}) - f(B)\Big) \geq 0\,.
    \end{align*}
\end{llemma}

We can now use this lemma to prove the following proposition.

\explorationSuff*

\begin{proof}
    We need to look for conditions for Eq.~\eqref{eq:pConditions} to be satisfied
\begin{align}
    p(\bar\beta_i-3\bar\alpha_i) \leq -\frac{g_i+\gamma_{T, \delta}}{\sqrt{\tau_i}}+(\beta_i-2\bar\alpha_i)\,, \hspace{1cm}
    p(3\bar\beta_i - \bar \alpha_i) \leq -\frac{g_i+\gamma_{T, \delta}}{\sqrt{\tau_i}}+\bar\beta_i\,,\tag{\ref{eq:pConditions}}
\end{align}
where $\gamma_{T, \delta}=3\sqrt{(2\sigma^2+c^2)(\log(dT/\delta)+\log(1+T))}$\,.

Considering the different configurations of $(\alpha,\ \beta)$ possible using Lemma~\ref{lem:marginal_sum}, gives $5$ zones with different sufficient conditions for the existence of a $p_i\in[0, 1]$ satisfying Eq.~\eqref{eq:pConditions}. They are summarized in Table~\ref{tab:exploThresholds}, and are upper-bounded by the DG-hardness defined in Definition~\ref{def:DGhardness}.

\begin{center}
\begin{minipage}{.58\textwidth}
\begin{table}[H]
    \centering
    \begin{tabular}{ccc}
        \toprule
        &Zone & Threshold of $\frac{\tau_i}{(g_{T, \delta}+\gamma_{T,\delta})^2}$ \\
        \midrule
        \circled{1}&$\bar \alpha_i\leq 0$,\ $\bar\beta_i>0$ &  $1/\bar\beta_i^2$\\
        \circled{2}&$0\leq \bar\alpha_i \leq \bar\beta_i/3$ & $1/(\bar \beta_i-2\bar\alpha_i)^2$\\
        \circled{3}&$ 0\leq \bar\beta_i/3 \leq \bar\alpha_i \leq 3\bar\beta_i $ & $(\bar\alpha_i+\bar\beta_i)^2/(\bar\beta_i-\bar\alpha_i)^4$\\
        \circled{4}&$ 0 \leq 3\bar\beta_i \leq \alpha_i$ & $1/(\bar\alpha_i-2\bar\beta_i)^2$\\
        \circled{5}&$ \bar\alpha_i>0$,\ $\bar\beta_i\leq 0$ & $1/\bar \alpha_i^2$ \\
        \bottomrule
    \end{tabular}
    \caption{Exploration thresholds for \UpdateExplo.}
    \label{tab:exploThresholds}
\end{table}
\end{minipage}
\hfill
\begin{minipage}{.4\textwidth}
\includegraphics[width = \textwidth]{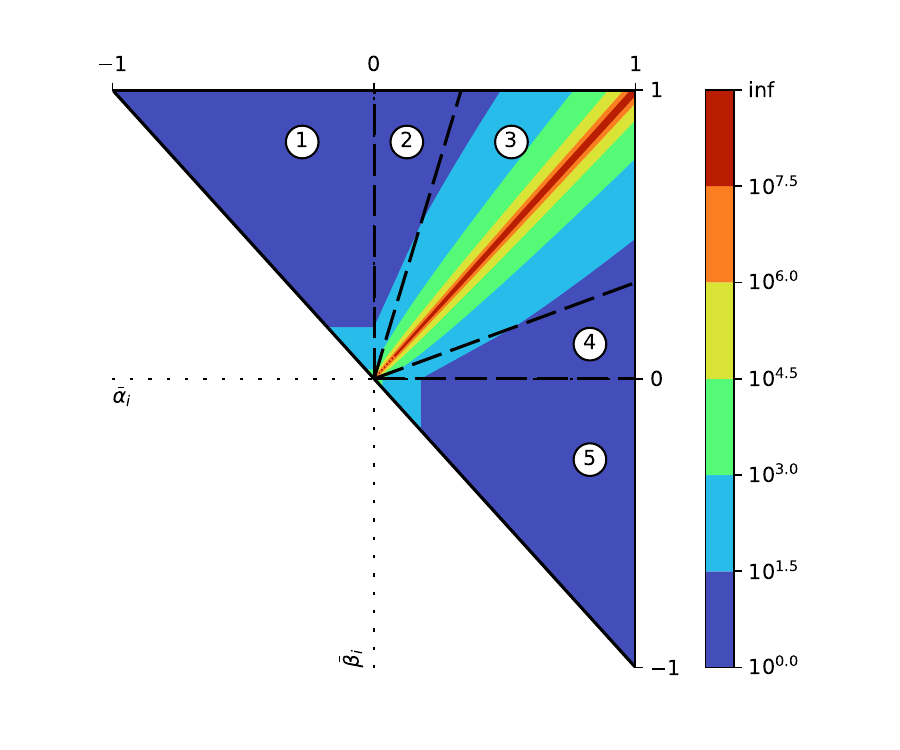}
\captionof{figure}{Exploration thresholds for Subroutine \UpdateExplo as a function of $\bar\alpha_i$ and $\bar\beta_i$ for $c=1$.}
\end{minipage}
\end{center}

\end{proof}

\end{document}